\documentclass{article}

\usepackage[final]{neurips_2018}

\usepackage[utf8]{inputenc} 
\usepackage[T1]{fontenc}    
\usepackage{url}            
\usepackage{booktabs}       
\usepackage{amsfonts}       
\usepackage{dsfont}         
\usepackage{nicefrac}       
\usepackage{microtype}      
\usepackage{xspace}         

\usepackage{graphicx}
\usepackage{subfigure}
\usepackage{wrapfig}
\usepackage{booktabs} 

\usepackage[usenames, dvipsnames]{xcolor} 
\usepackage{tikz}
\usetikzlibrary{arrows, er, positioning}
\usepackage{tikz-cd}

\usepackage{amsmath, amsthm, amssymb}
\usepackage{enumerate}

\usepackage{stmaryrd} 

\usepackage{natbib}

\usepackage{algorithm}
\usepackage{algorithmic}
\renewcommand{\algorithmicrequire}{\textbf{Input:}}
\renewcommand{\algorithmicensure}{\textbf{Output:}}

\theoremstyle{plain}
\newtheorem{theorem}{Theorem}
\newtheorem{lemma}{Lemma}
\newtheorem{proposition}{Proposition}

\theoremstyle{definition}
\newtheorem{definition}{Definition}
\newtheorem{problem}{Problem}

\newcommand{\Sec}[1]{Section~\ref{sec:#1}}
\newcommand{\Eq}[1]{Eq.~(\ref{eq:#1})}
\newcommand{\Alg}[1]{Alg.~(\ref{alg:#1})}
\newcommand{\Fig}[1]{Figure~\ref{fig:#1}}
\newcommand{\Definition}[1]{Definition~\ref{def:#1}}
\newcommand{\Theorem}[1]{Theorem~\ref{th:#1}}
\newcommand{\Proposition}[1]{Proposition~\ref{prop:#1}}

\newcommand{\Lemma}[1]{Lemma~\ref{lem:#1}}
\newcommand{\Problem}[1]{Problem~\ref{problem:#1}}

\newcommand{\BEAS}{\begin{eqnarray*}}
\newcommand{\EEAS}{\end{eqnarray*}}
\newcommand{\BEA}{\begin{eqnarray}}
\newcommand{\EEA}{\end{eqnarray}}
\newcommand{\BEQ}{\begin{equation}}
\newcommand{\EEQ}{\end{equation}}
\newcommand{\BEQS}{\begin{equation*}}
\newcommand{\EEQS}{\end{equation*}}
\newcommand{\BIT}{\begin{itemize}}
\newcommand{\EIT}{\end{itemize}}
\newcommand{\BNUM}{\begin{enumerate}}
\newcommand{\ENUM}{\end{enumerate}}
\newcommand{\BA}{\begin{array}}
\newcommand{\EA}{\end{array}}

\newcommand{\set}[1]{\llbracket#1\rrbracket}
\newcommand{\diag}{\mathop{\rm diag}}

\newcommand{\scprod}[2]{\langle#1,#2\rangle}

\def \E{{\mathbb E}}

\def \R{{\mathbb R}}

\newcommand{\Exp}[1]{\E\left[#1\right]}

\def \rt{\widetilde{r}}

\newcommand{\ie}{i.e.\ }
\newcommand{\eg}{e.g.\ }
\newcommand{\wrt}{w.r.t.\ }

\newcommand{\NP}{\mathbf{NP}}
\newcommand{\lipcst}{\mathbf{LIP\textrm{-}CST}\xspace}

\newcommand{\Sigmat}{\widetilde{\Sigma}}

\newcommand{\suchthat}{:}

\DeclareMathOperator*{\maximize}{maximize}
\DeclareMathOperator*{\st}{s.t.}
\DeclareMathOperator{\jac}{D}

\newcommand{\lip}[1]{L(#1)}

\title{Lipschitz regularity of deep neural networks:\\analysis and efficient estimation}
\author{
  Kevin Scaman\\
  Huawei Noah's Ark Lab\\
  \texttt{kevin.scaman@huawei.com}\\
  \And
  Aladin Virmaux\\
  Huawei Noah's Ark Lab\\
  \texttt{aladin.virmaux@huawei.com}\\
}

\begin{document}

\maketitle

\begin{abstract}
Deep neural networks are notorious for being sensitive to small well-chosen
perturbations, and estimating the regularity of such architectures is of utmost
importance for safe and robust practical applications.
In this paper, we investigate one of the key characteristics to assess the
regularity of such methods: the Lipschitz constant of deep learning
architectures.
First, we show that, even for two layer neural networks, the exact computation
of this quantity is NP-hard and state-of-art methods may significantly
overestimate it.
Then, we both extend and improve previous estimation methods by providing
\emph{AutoLip}, the first generic algorithm for upper bounding the Lipschitz
constant of any automatically differentiable function.
We provide a power method algorithm working with automatic differentiation,
allowing efficient computations even on large convolutions.
Second, for sequential neural networks, we propose an improved algorithm named
\emph{SeqLip} that takes advantage of the linear computation graph to split the
computation per pair of consecutive layers.
Third we propose heuristics on \emph{SeqLip} in order to tackle very large networks.
Our experiments show that \emph{SeqLip} can significantly improve on the
existing upper bounds.
Finally, we provide an implementation of \emph{AutoLip} in the \emph{PyTorch}
environment that may be used to better estimate the robustness of a given
neural network to small perturbations or regularize it using more precise
Lipschitz estimations.
\end{abstract}

\section{Introduction}
\label{sec:intro}

Deep neural networks made a striking entree in machine learning and quickly became
state-of-the-art algorithms in many tasks such as computer
vision \citep{krizhevsky2012imagenet,szegedy2016rethinking,he2016deep,huang2017densely},
speech recognition and generation \citep{graves2014towards,van2016wavenet} or
natural language processing \citep{mikolov2013linguistic, vaswani2017attention}.

However, deep neural networks are known for being very sensitive to their input, and \emph{adversarial examples} provide a good illustration of their lack of robustness \citep{szegedy2013intriguing, goodfellow2014explaining}. Indeed, a well-chosen small perturbation
of the input image can mislead a neural network and significantly decrease its classification accuracy.
One metric to assess the robustness of neural networks to small perturbations
is the \emph{Lipschitz constant} (see \Definition{lipschitz}), which upper
bounds the relationship between input perturbation and output variation for a
given distance.
For generative models, the recent \emph{Wasserstein
  GAN}~\citep{2017arXiv170107875A} improved the training stability of GANs by
reformulating the optimization problem as a minimization of the Wasserstein
distance between the real and generated
distributions~\citep{villani2008optimal}.
However, this method relies on an efficient way of constraining the Lipschitz
constant of the critic, which was only partially addressed in the original
paper, and the object of several follow-up works
\citep{miyato2018spectral,gulrajani2017improved}.

Recently, Lipschitz continuity was used in order to improve the
state-of-the-art in several deep learning topics:
(1) for robust learning, avoiding adversarial attacks was achieved in
\citep{weng2018evaluating} by constraining local Lipschitz constants in neural
networks.
(2) For generative models, using spectral normalization on each layer allowed
\citep{miyato2018spectral} to successfully train a GAN on ILRSVRC2012 dataset.
(3) In deep learning theory, novel generalization bounds critically rely on the Lipschitz
constant of the neural
network~\cite{Luxburg:2004:DCL:1005332.1005357,DBLP:conf/nips/BartlettFT17,DBLP:conf/nips/NeyshaburBMS17}.

To the best of our knowledge, the first upper-bound on the
Lipschitz constant of a neural network was described in ~\cite[Section
4.3]{szegedy2013intriguing}, as the product of the spectral norms of linear
layers (a special case of our generic algorithm, see
\Proposition{upper.bound.spectral}).
More recently, the Lipschitz constant of \emph{scatter networks} was analyzed
in \cite{Balan2017LipschitzPF}.
Unfortunately, this analysis does not extend to more general architectures.

Our aim in this paper is to provide a rigorous and practice-oriented study on
how Lipschitz constants of neural networks and automatically differentiable
functions may be estimated.
We first precisely define the notion of
Lipschitz constant of vector valued
functions in \Sec{definitions}, and then show in \Sec{nphard} that its
estimation is, even for 2-layer \emph{Multi-Layer-Perceptrons} (MLP),
$\NP$-hard.
In \Sec{autolip}, we both extend and improve previous estimation methods
by providing \emph{AutoLip}, the first generic algorithm for upper bounding the
Lipschitz constant of any automatically differentiable function.
Moreover, we show how the Lipschitz constant of most neural network layers may
be computed efficiently using automatic differentiation
algorithms~\citep{rall1981automatic} and libraries such as
PyTorch~\citep{pytorch}.
Notably, we extend the power method to convolution layers using automatic
differentiation to speed-up the computations.
In \Sec{sequential}, we provide a theoretical analysis of AutoLip in the case
of sequential neural networks, and show that the upper bound may lose a
multiplicative factor \emph{per activation layer}, which may significantly
downgrade the estimation quality of AutoLip and lead to a very large and
unrealistic upper bound.
In order to prevent this, we propose an improved
algorithm called \emph{SeqLip} in the case of sequential neural networks, and
show in \Sec{exps} that SeqLip may significantly improve on AutoLip.
Finally we discuss the different algorithms on the
AlexNet~\citep{krizhevsky2012imagenet} neural network for computer vision using
the proposed algorithms.
\footnote{The code used in this paper is available at
\url{https://github.com/avirmaux/lipEstimation}.}

\section{Background and notations}
\label{sec:definitions}

In the following, we denote as $\scprod{x}{y}$ and $\|x\|_2$ the scalar product
and $L_2$-norm of the Hilbert space $\R^n$, $x\cdot y$ the coordinate-wise
product of $x$ and $y$, and $f\circ g$ the composition between the functions
$f:\R^k\rightarrow\R^m$ and $g:\R^n\rightarrow\R^k$. For any differentiable function
$f:\R^n\rightarrow\R^m$ and any point $x\in\R^n$, we will denote as $\jac_x
f\in\R^{m\times n}$ the differential operator of $f$ at $x$, also called the
\emph{Jacobian matrix}.
Note that, in the case of real valued functions (\ie $m=1$), the gradient of
$f$ is the transpose of the differential operator:
$\nabla f(x) = (\jac_x f)^\top$. Finally, $\diag_{n,m}(x)\in\R^{n\times m}$ is the rectangular matrix with $x\in\R^{\min\{n,m\}}$
along the diagonal and $0$ outside of it. When unambiguous, we will use the
notation $\diag(x)$ instead of $\diag_{n,m}(x)$. All proofs are available as
supplemental material.

\begin{definition}
  \label{def:lipschitz}
A function $f: \R^n \rightarrow \R^m$ is called \emph{Lipschitz continuous} if there exists a constant
$L$ such that
\begin{align*}
  \forall x, y \in \R^n,\ \|f(x) - f(y)\|_2 \leq L\,\|x - y\|_2.
\end{align*}
The smallest $L$ for which the previous inequality is true is called the
\emph{Lipschitz constant} of $f$ and will be denoted $\lip{f}$.
\end{definition}
For locally Lipschitz functions (\ie functions whose restriction to some
neighborhood around any point is Lipschitz), the Lipschitz constant may be
computed using its differential operator.
\begin{theorem}[{Rademacher~\cite[Theorem 3.1.6]{federer2014geometric}}]
  \label{th:definitions.lipschitz.grad}
  If $f: \R^n \rightarrow \R^m$ is a locally Lipschitz continuous function, then $f$ is
  differentiable almost everywhere. Moreover, if $f$ is Lipschitz continuous, then
  \begin{align}
    \label{eq:optimization.lipschitz}
    \lip{f} &= \sup_{x\in\R^n} \| \jac_x f \|_2
  \end{align}
  where $\|M\|_2 = \sup_{\{x~\suchthat~\|x\|=1\}} \|M x\|_2$ is the operator
  norm of the matrix $M\in\R^{m\times n}$.
\end{theorem}
In particular, if $f$ is real valued (\ie $m=1$), its
Lipschitz constant is the maximum norm of its gradient $\lip{f} = \sup_x
\|\nabla f(x)\|_2$ on its domain set.
Note that the supremum in \Theorem{definitions.lipschitz.grad} is a slight
abuse of notations, since the differential $\jac_x f$ is defined \emph{almost
  everywhere} in $\R^n$, except for a set of Lebesgue measure zero.

\section{Exact Lipschitz computation is NP-hard}
\label{sec:nphard}
In this section, we show that the exact computation of the Lipschitz constant of neural networks is $\NP$-hard, hence motivating the need for good approximation algorithms. More precisely, upper bounds are in this case more valuable as they ensure that the variation of the function, when subject to an input perturbation, remains small.
A \emph{neural network} is, in essence, a succession of linear operators and non-linear activation functions.
The most simplistic model of neural network is the
\emph{Multi-Layer-Perceptron} (MLP) as defined below.

\begin{definition}[MLP]
\label{def:MLP}
A $K$-layer
\emph{Multi-Layer-Perceptron} $f_{MLP}:\R^n\rightarrow\R^m$
is the function
\begin{align*}
  f_{MLP}(x) &= T_K \circ \rho_{K-1} \circ \cdots \circ \rho_1 \circ T_1 (x),
\end{align*}
where $T_k: x\mapsto M_k x + b_k$ is an affine function and $\rho_k: x\mapsto (g_k(x_i))_{i\in\set{1,n_k}}$ is a non-linear activation function.
\end{definition}
Many standard deep network architectures (\textit{e.g.} CNNs) follow --to some
extent-- the MLP structure.
It turns out that even for $2$-layer MLPs, the computation of the Lipschitz
constant is $\NP$-hard.

\begin{problem}[$\lipcst$]
\label{problem:lipcst}
$\lipcst$ is the decision problem associated to the exact computation of the Lipschitz constant of a $2$-layer MLP with ReLU activation layers.
\BIT
\item[]\textbf{Input:} Two matrices $M_1\in\R^{l\times n}$ and $M_2\in\R^{m\times l}$, and a constant $\ell\geq 0$.
\item[]\textbf{Question:} Let $f = M_2 \circ \rho \circ M_1$ where $\rho(x) = \max\{0, x\}$ is the ReLU activation function. Is the Lipschitz constant $\lip{f} \leq \ell$ ?
\EIT
\end{problem}
\Theorem{lipcst.np.hardness} shows that, even for extremely simple neural
networks, exact Lipschitz computation is not achievable in polynomial time
(assuming that $\mathbf{P}\neq\NP$). The proof of
\Theorem{lipcst.np.hardness} is available in the supplemental material.
\begin{theorem}
  \label{th:lipcst.np.hardness}
  \Problem{lipcst} is $\NP$-hard.
\end{theorem}
\Theorem{lipcst.np.hardness} relies on a reduction to the $\NP$-hard problem of quadratic concave minimization on a hypercube by considering well-chosen matrices $M_1$ and $M_2$.

\section{AutoLip: a Lipschitz upper bound through automatic differentiation}
\label{sec:autolip}

Efficient implementations of backpropagation in modern deep learning libraries
such as \emph{PyTorch}~\citep{pytorch} or
\emph{TensorFlow}~\citep{tensorflow2015-whitepaper} rely on on the
concept of \emph{automatic differentiation} \cite{griewank2008evaluating, rall1981automatic}.
Simply put, automatic differentiation is a principled approach to the
computation of gradients and differential operators of functions resulting from
$K$ successive operations.

\begin{definition}
  \label{def:computable.function}
  A function $f:\R^n \rightarrow \R^m$ is \emph{computable in $K$ operations}
  if it is the result of $K$ \emph{simple} functions in the following way:
  $\exists (\theta_1, ...,\theta_K)$ functions of the input $x$ and $(g_1,
  \dots, g_K)$ where $g_l$ is a function of $(\theta_i)_{i \leq l-1}$
  such that:
  \begin{align}
    \theta_0(x) = x,\hspace{5ex} \theta_K(x) = f(x),\hspace{5ex}
    \forall k \in \set{1, K},\ \theta_k(x) = g_k(x, \theta_1(x), \dots,
    \theta_{k-1}(x))\,.
  \end{align}
\end{definition}

\begin{figure}[h]
  \centering
  \begin{tikzpicture}[>=latex,
  base/.style={circle, draw},
  theta/.style={rectangle, draw},
  cst/.style={ellipse, inner sep=1pt, draw}]
  \node[theta] (theta0) at (0, 1) {$\theta_0=x$};
  \node[theta] (theta1) at (2.5, 0) {$\theta_1=\theta_0/2$};
  \node[cst] (w) at (0, 2.8) {$\theta_2 = \omega$};

  \node[theta] (theta3) at (2.5, 2) {$\theta_3 = \sin(\theta_0)$};
  \node[theta] (theta4) at (5.5, 1) {$\theta_4 = \theta_1 - \theta_2\theta_3$};
  \node[theta] (theta5) at (7, 0) {$\theta_5 = \ln(1+e^{\theta_1})$};
  \node[theta] (theta6) at (8.8, 1) {$\theta_6 = |\theta_4|$};
  \node[theta] (theta7) at (12.5, 1) {$\theta_7 = \theta_5 + \theta_6$};

  \draw[->] (theta0) edge node[near end, below] {$g_1$} (theta1.west);
  \draw[->] (theta0) edge node[near end, below] {$g_3$} (theta3.west);
  \draw[->] (theta1) edge node[near end, above, xshift=3pt] {$g_5$} (theta5.west);
  \draw[->] (theta3) edge node[near end, below] {} (theta4.west);
  \draw[->] (theta1) edge node[near end, above, xshift=-5pt] {$g_4$} (theta4.west);
  \draw[->] (theta4) edge node[near end, above] {$g_6$} (theta6);
  \draw[->] (theta6) edge node[near end, above, xshift=4pt] {$g_7$} (theta7);
  \draw[->] (theta5) edge[bend right=10] (theta7.west);

  \draw[->] (w) edge[bend left] (theta4.north);

\end{tikzpicture}
  \caption{Example of a computation graph for $f_\omega(x) = \ln(1+e^{x/2}) + |x/2 - \omega \sin(x)|$.}
  \label{fig:comp_graph}
\end{figure}
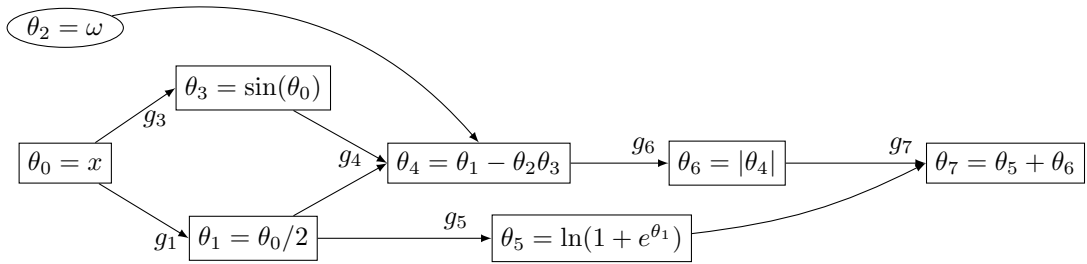

We assume that these operations are all locally Lipschitz-continuous,
and that their partial derivatives $\partial_i g_k(x)$ can be computed and efficiently maximized.
This assumption is discussed in \Sec{liplayers} for the main operations used in neural networks.
When the function is real valued (\ie $m=1$), the backpropagation algorithm allows to compute
its gradient efficiently in time proportional to the number of operations $K$~\citep{linnainmaa1970representation}.
For the computation of the Lipschitz constant $\lip{f}$, a forward propagation through the computation graph is sufficient.
\begin{algorithm}[t]
  \caption{AutoLip}
  \label{alg:AutoLip}
  \begin{algorithmic}[1]
    \REQUIRE{function $f:\R^n\rightarrow\R^m$ and its computation graph
      $(g_1,...,g_K)$}
		\ENSURE{upper bound on the Lipschitz constant: $\hat{L}_{AL} \geq \lip{f}$}
			\STATE $\mathcal{Z} = \{(z_0,...,z_K)~:~\forall k\in\set{0,K}, \theta_k\mbox{ is
  constant}\Rightarrow z_k = \theta_k(0)\}$
    \STATE $L_0 \gets 1$
    \FOR{$k=1$ to $K$}
      \STATE $L_k \gets \sum\limits_{i=1}^{k-1}
      \max\limits_{z \in \mathcal{Z}}
      \|\partial_i g_k(z) \|_2 L_i$
    \ENDFOR
  \STATE \textbf{return} $\hat{L}_{AL} = L_k$
  \end{algorithmic}
\end{algorithm}
More specifically, the chain rule immediately implies
\begin{align}
  \jac_x\theta_k &= \sum_{i=1}^{k-1}  \partial_i g_k(\theta_0(x), \dots, \theta_{k-1}(x))
    \jac_x\theta_i\,,
\end{align}
and taking the norm then maximizing over all possible values of $\theta_i(x)$
leads to the \emph{AutoLip} algorithm described in \Alg{AutoLip}. This algorithm is an extension of the well known
product of operator norms for MLPs (see \eg \cite{miyato2018spectral}) to any function computable in $K$ operations.

\begin{proposition}
  \label{prop:upper.bound.spectral}
  For any MLP (see \Definition{MLP}) with $1$-Lipschitz activation functions
  (\eg ReLU, Leaky ReLU, SoftPlus, Tanh, Sigmoid, ArcTan or Softsign), the
  AutoLip upper bound becomes
\BEQS
\hat{L}_{AL} = \prod_{k=1}^K \|M_k\|_2.
\EEQS
\end{proposition}

Note that, when the intermediate function $\theta_k$
does not depend on $x$, it is not necessary to take a maximum over all possible
values of $\theta_k(x)$. To this end we define the set of feasible intermediate
values as
\BEQ
\mathcal{Z} = \{(z_0,...,z_K)~:~\forall k\in\set{0,K}, \theta_k\mbox{ is
  constant}\Rightarrow z_k = \theta_k(0)\},
\EEQ
and only maximize partial derivatives over this set.
In practice, this is equivalent to removing branches of the computation
graph that are not reachable from node $0$ and replacing them by constant
values.
To illustrate this
definition, consider a simple matrix product operation $f(x) = Wx$. One possible computation graph for $f$ is $\theta_0 = x$, $\theta_1 = W$ and
$\theta_2 = g_2(\theta_0, \theta_1) = \theta_1 \theta_0$. While the quadratic function $g_2$ is not
Lipschitz-continuous, its derivative \wrt $\theta_0$ is bounded by $\partial_0 g_2(\theta_0, \theta_1) = \theta_1 = W$. Since $\theta_1$ is
constant relatively to $x$, we have $\mathcal{Z} = \{(x, 0)\}$ and the
algorithm returns the exact Lipschitz constant $\hat{L}_{AL} = \lip{f} = \|W\|_2$.

\paragraph{Example.}
We consider the graph explicited on~\Fig{comp_graph}. Since $\theta_2$
is a constant \wrt $x$, we can replace it by its value $\omega$ in all other nodes. Then, the AutoLip algorithm runs as follows:
\BEQ
\hat{L}_{AL} = L_7 = L_6 + L_5 = L_1 + L_4 = 2L_1 + w L_3 = 1 + \omega.
\EEQ
Note that, in this example, the Lipschitz upper bound $\hat{L}_{AL}$ matches the exact Lipschitz constant $\lip{f_\omega} = 1 + \omega$.

\section{Lipschitz constants of typical neural network layers}
\label{sec:liplayers}

\paragraph{Linear and convolution layers.}
The Lipschitz constant of an affine function $f:x \mapsto Mx + b$ where $M\in\R^{m\times n}$ and $b\in\R^m$ is the largest singular value of its associated matrix $M$, which may be computed efficiently, up to a given precision, using the \emph{power method} \citep{mises1929praktische}. In the case of convolutions, the associated matrix may be difficult to access and high dimensional, hence making the direct use of the power method impractical.
To circumvent this difficulty, we extend the power method to any affine function on whose automatic differentiation can be used (\eg linear or convolution layers of neural networks) by noting that the only matrix multiplication of the power method $M^\top M x$ can be computed by differentiating a well-chosen function.

 \begin{lemma}
   Let $M\in\R^{m\times n}$, $b\in\R^m$ and $f:x \mapsto Mx + b$ be an affine
   function. Then, for all $x\in\R^n$, we have
   \BEQS
		\BA{lll}
     M^{\top}M x &=& \nabla g(x)\,,
   \EA
   \EEQS
	where $g(x) = \frac{1}{2} \|f(x) - f(0)\|_2^2$.
 \end{lemma}
\begin{proof}
By definition, $g(x) = \frac{1}{2}\|Mx\|_2^2$, and differentiating this equation leads to the desired result.
\end{proof}

\begin{algorithm}[h]
 \caption{AutoGrad compliant power method}
 \label{alg:autoPowerIteration}
 \begin{algorithmic}[1]
   \REQUIRE{affine function $f: \R^n\rightarrow\R^m$, number of iteration $N$}
   \ENSURE{approximation of the Lipschitz constant $\lip{f}$}
   \FOR{$k = 1$ to $N$}
     \STATE $v \gets \nabla g(v)$ where $g(x) = \frac{1}{2}\|f(x) - f(0)\|_2^2$
     \STATE $\lambda \gets \|v\|_2$
     \STATE $v \gets v/ \lambda$
   \ENDFOR
   \STATE \textbf{return} $\lip{f} = \|f(v) - f(0)\|_2$
 \end{algorithmic}
\end{algorithm}

The full algorithm is described in \Alg{autoPowerIteration}.
Note that this algorithm is fully compliant with any dynamic graph deep
learning libraries such as PyTorch.  The gradient of the square norm may be
computed through autograd, and the gradient of $\lip{f}$ may be computed the
same way without any more programming
effort.
Note that the gradients \wrt $M$ may also be computed with the closed form
formula $\nabla_M \sigma = u_1 v_1^{\top}$ where $u_1$ and $v_1$ are respectively
the left and right singular vector of $M$ associated to the singular
value $\sigma$~\citep{magnus1985differentiating}.
The same algorithm may be straightforwardly iterated to compute the $k$-largest
singular values.

\paragraph{Other layers.}
Most activation functions such as ReLU, Leaky ReLU, SoftPlus, Tanh, Sigmoid,
ArcTan or Softsign, as well as max-pooling, have a
Lipschitz constant equal to $1$.
Other common neural network layers such as dropout, batch normalization and
other pooling methods all have simple and explicit Lipschitz constants. We
refer the reader to \eg \cite{Goodfellow-et-al-2016} for more information on
this subject.

\section{Sequential neural networks}
\label{sec:sequential}
Despite its generality, AutoLip may be subject to large errors due to the
multiplication of smaller errors at each iteration of the algorithm. In this
section, we improve on the AutoLip upper bound by a more refined analysis of
deep learning architectures in the case of MLPs.
More specifically, the Lipschitz constant of MLPs have an explicit formula using \Theorem{definitions.lipschitz.grad} and the
chain rule:
\BEQ\label{eq:LipMLP}
\lip{f_{MLP}} = \sup_{x\in\R^n} \|M_K\diag(g'_{K-1}(\theta_{K-1}))M_{k-1}...M_2\diag(g'_1(\theta_1))M_1\|_2,
\EEQ
where $\theta_k = T_k \circ \rho_{k-1} \circ \cdots \circ \rho_1 \circ T_1(x)$ is the intermediate output after $k$ linear layers.

Considering \Proposition{upper.bound.spectral} and \Eq{LipMLP}, the equality $\hat{L}_{AL} =
\lip{f_{MLP}}$ only takes place if all activation layers
$\diag(g'_k(\theta_k))$ map the first singular vector of $M_k$ to the first
singular vector of $M_{k+1}$ by Cauchy-Schwarz inequality.
However, differential operators of activation layers, being diagonal matrices,
can only have a limited effect on input vectors, and in practice, first
singular vectors will tend to misalign, leading to a drop in the Lipschitz
constant of the MLP. This is the intuition behind \emph{SeqLip}, an improved
algorithm for Lipschitz constant estimation for MLPs.

\subsection{SeqLip, an improved algorithm for MLPs}
\label{sec:seqlip}

In \Eq{LipMLP}, the diagonal matrices $\diag(g'_{K-1}(\theta_{K-1}))$ are
difficult to evaluate, as they may depend on the input value $x$ and previous
layers. Fortunately, as stated in \Sec{liplayers}, most
major activation functions are $1$-Lipschitz. More specifically, these
activation functions have a derivative $g'_k(x)\in[0,1]$. Hence, we may replace
the supremum on the input vector $x$ by a supremum over all possible values:

\begin{align}\label{eq:ub_sigma}
\lip{f_{MLP}} \leq
  \max_{\forall i,\ \sigma_i \in [0, 1]^{n_i}} \|M_K \diag(\sigma_{K-1}) \cdots \diag(\sigma_1) M_1 \|_2\,,
\end{align}
where $\sigma_i$ corresponds to all possible derivatives of the activation gate.
Solving the right hand side of \Eq{ub_sigma} is still a hard problem, and the high dimensionality of the search space $\sigma\in[0,1]^{\sum_{i=1}^K n_i}$ makes purely combinatorial approaches prohibitive even for small neural networks. In order to decrease the complexity of the problem, we split the operator norm in $K-1$ parts using the
SVD decomposition of each matrix $M_i = U_i \Sigma_i
V_i^{\top}$ and the submultiplicativity of the operator norm:
\begin{align*}
  \begin{split}
    \lip{f_{MLP}} &\leq
    \max_{\forall i,\ \sigma_i \in [0, 1]^{n_i}} \|
    \Sigma_K V_K^{\top} \diag(\sigma_K)
    U_{K-1} \Sigma_{K-1} V_{K-1}^{\top}
    \diag(\sigma_{K-1})
    \dots
    \diag(\sigma_1) U_1 \Sigma_1 \|_2\,,
    \\
    & \leq \prod_{i=1}^{K-1} \max_{\sigma_i\in [0, 1]^{n_i}}
    \left\|
      \widetilde{\Sigma}_{i+1} V_{i+1}^{\top} \diag(\sigma_{i+1}) U_i
      \widetilde{\Sigma}_{i}
    \right\|_2\,,
  \end{split}
\end{align*}
where $\widetilde{\Sigma}_i = \Sigma_i$ if $i\in\{1, K\}$ and $\widetilde{\Sigma}_i =
\Sigma_i^{1/2}$ otherwise. Each activation layer can now be solved independently, leading to the \emph{SeqLip} upper bound:
\BEQ
\hat{L}_{SL} = \prod_{i=1}^{K-1} \max_{\sigma_i\in [0, 1]^{n_i}}
    \left\|
      \widetilde{\Sigma}_{i+1} V_{i+1}^{\top} \diag(\sigma_{i+1}) U_i
      \widetilde{\Sigma}_{i}
    \right\|_2\,.
\EEQ
When the activation layers are ReLU and the inner layers are small ($n_i\leq 20$), the gradients are $g_k'\in\{0,1\}$ and we may explore the entire search space $\sigma_i\in\{0,1\}^{n_i}$ using a brute force combinatorial approach.
Otherwise, a gradient ascent may be used by computing gradients via the power
method described in Alg.~2. In our experiments, we call this heuristic \emph{Greedy SeqLip}, and verified that the
incurred error is at most $1\%$ whenever the exact optimum is computable.
Finally, when the dimension of the layer is too large to compute a whole SVD,
we perform a low rank-approximation of the matrix $M_i$ by retaining the first
$E$ eigenvectors ($E = 200$ in our experiments).

\subsection{Theoretical analysis of SeqLip}
\label{sec:theory}
In order to better understand how SeqLip may improve on AutoLip, we now consider a simple setting in which all linear layers have a large difference between their first and second singular values.
For simplicity, we also assume that activation functions have a derivative
$g'_k(x) \in [0,1]$, although the following results easily generalize as long as the
derivative remains bounded.
Then, the following theorem holds.

\begin{theorem}\label{th:approxSeqlip}
Let $M_k$ be the matrix associated to the $k$-th linear layer, $u_k$ (resp. $v_k$) its first left (resp. right) singular vector, and $r_k = s_{k,2} / s_{k,1}$ the ratio between its second and first singular values. Then, we have
\BEQS
\hat{L}_{SL} \leq \hat{L}_{AL} \prod_{k=1}^{K-1} \sqrt{(1 - r_k -
  r_{k+1})\max_{\sigma\in[0,1]^{n_k}} \scprod{\sigma\cdot v_{k+1}}{u_k}^2 + r_k
  + r_{k+1} + r_k r_{k+1}}\,.
\EEQS
\end{theorem}

Note that $\max_{\sigma\in[0,1]^{n_k}} \scprod{\sigma\cdot v_{k+1}}{u_k}^2 \leq 1$ and, when the ratios $r_k$ are negligible, then
\BEQ
\hat{L}_{SL} \leq \hat{L}_{AL} \prod_{k=1}^{K-1} \max_{\sigma\in[0,1]^{n_k}} \left|\scprod{\sigma\cdot v_{k+1}}{u_k}\right|.
\EEQ
Intuitively, each activation layer may align $u_k$ to $v_{k+1}$ only to a certain extent.
Moreover, when the two singular vectors $u_k$ and $v_{k+1}$ are not too
similar, this quantity can be substantially smaller than $1$.
To illustrate this idea, we now show that $\max_{\sigma\in[0,1]^{n_k}}
|\scprod{\sigma\cdot v_{k+1}}{u_k}|$ is of the order of $1/\pi$ if the two
vectors are randomly chosen on the unit sphere.

\begin{lemma}
\label{lem:misalign}
Let $x\geq 0$ and $u,v\in\R^n$ be two independent random vectors
taken uniformly on the unit sphere $\mathbb{S}^{n-1} =
\{x\in\R^n~:~\|x\|_2=1\}$. Then we have
\begin{equation*}
  \begin{tikzcd}[row sep=large, column sep=large]
    \displaystyle\max_{\sigma\in[0,1]^n} |\scprod{\sigma\cdot u}{v}| \arrow[r, "n \rightarrow
    +\infty"{below}]&
    \displaystyle\frac{1}{\pi} \hspace{3ex}\text{almost surely.}
  \end{tikzcd}
\end{equation*}
\end{lemma}

Intuitively, when the ratios between the second and first singular values are sufficiently small, each activation layer decreases the Lipschitz constant by a factor $1/\pi$ and
\BEQ
\hat{L}_{SL} \approx \frac{\hat{L}_{AL}}{\pi^{K-1}}.
\EEQ
For example, for $K=5$ linear layers, we have $\pi^{K-1} \approx 100$ and a
large improvement may be expected for SeqLip compared to AutoLip. Of course, in
a more realistic setting, the eigenvectors of different layers are not
independent and, more importantly, the ratio between second and first eigenvalues may not be sufficiently
small.
However, this simple setting provides us with the best improvement one
can hope for, and our experiments in \Sec{exps} shows that at least part of the
suboptimality of AutoLip is due to the misalignment of eigenvectors.

\section{Experimentations}
\label{sec:exps}

As stated in \Theorem{lipcst.np.hardness}, computing the Lipschitz constant is an $\NP$-hard
problem. However, in low dimension (\eg $d \leq 5$), optimizing the problem in
\Eq{optimization.lipschitz} can be performed efficiently using a simple grid search.
This will provide a baseline to compare the different estimation algorithms.
In high dimension,
grid search is intractable and we consider several other estimation methods:
(1) grid search for \Eq{optimization.lipschitz}, (2) simulated annealing for
\Eq{optimization.lipschitz}, (3) product of Frobenius norms of linear layers
\citep{miyato2018spectral}, (4) product of spectral norms
\citep{miyato2018spectral} (equivalent to AutoLip in the case of MLPs).
Note that, for MLPs with ReLU activations, first order optimization methods
such as SGD are not usable because the function to optimize in
\Eq{optimization.lipschitz} is piecewise constant.
Methods (1) and (2) return lower bounds while (3) and (4) return upper bounds
on the Lipschitz constant.

\paragraph{Ideal scenario.}
We first show the improvement of SeqLip over AutoLip in an ideal setting where
inner layers have a low eigenvalue ratio $r_k$ and uncorrelated leading
eigenvectors.
To do so, we construct an MLP with weight matrices $M_i = U_i\diag(\lambda)
V_i^\top$ such that $U_i, V_i$ are random orthogonal matrices and $\lambda_1 =
1, \lambda_{i>1} = r$ where $r\in[0, 1]$ is the ratio between first and second
eigenvalue. \Fig{seqlip.n.layers} shows the decrease of SeqLip as the number
of layers of the MLP increases (each layer has $100$ neurons).
The theoretical limit is tight for small eigenvalue ratio.
Note that the AutoLip upper bound is
always $1$ as, by construction of the network, all layers have a spectral radius equal to one.

\paragraph{MLP.} We construct a $2$-dimensional dataset from a Gaussian Process
with RBF Kernel with mean $0$ and variance $1$.
We use $15000$ generated points as a synthetic dataset. An example of such a
dataset may be seen in \Fig{gp.dataset}.
We train MLPs of several depths with $20$ neurons at each layer,
on the synthetic dataset with MSE loss and ReLU activations.
Note that in all simulations, the greedy SeqLip algorithm is within a
$0.01\%$ error compared to SeqLip, which justify
its usage in higher dimension.

\begin{figure}[t]
\begin{minipage}{.5\textwidth}
  \centering
  \includegraphics[width=1\linewidth]{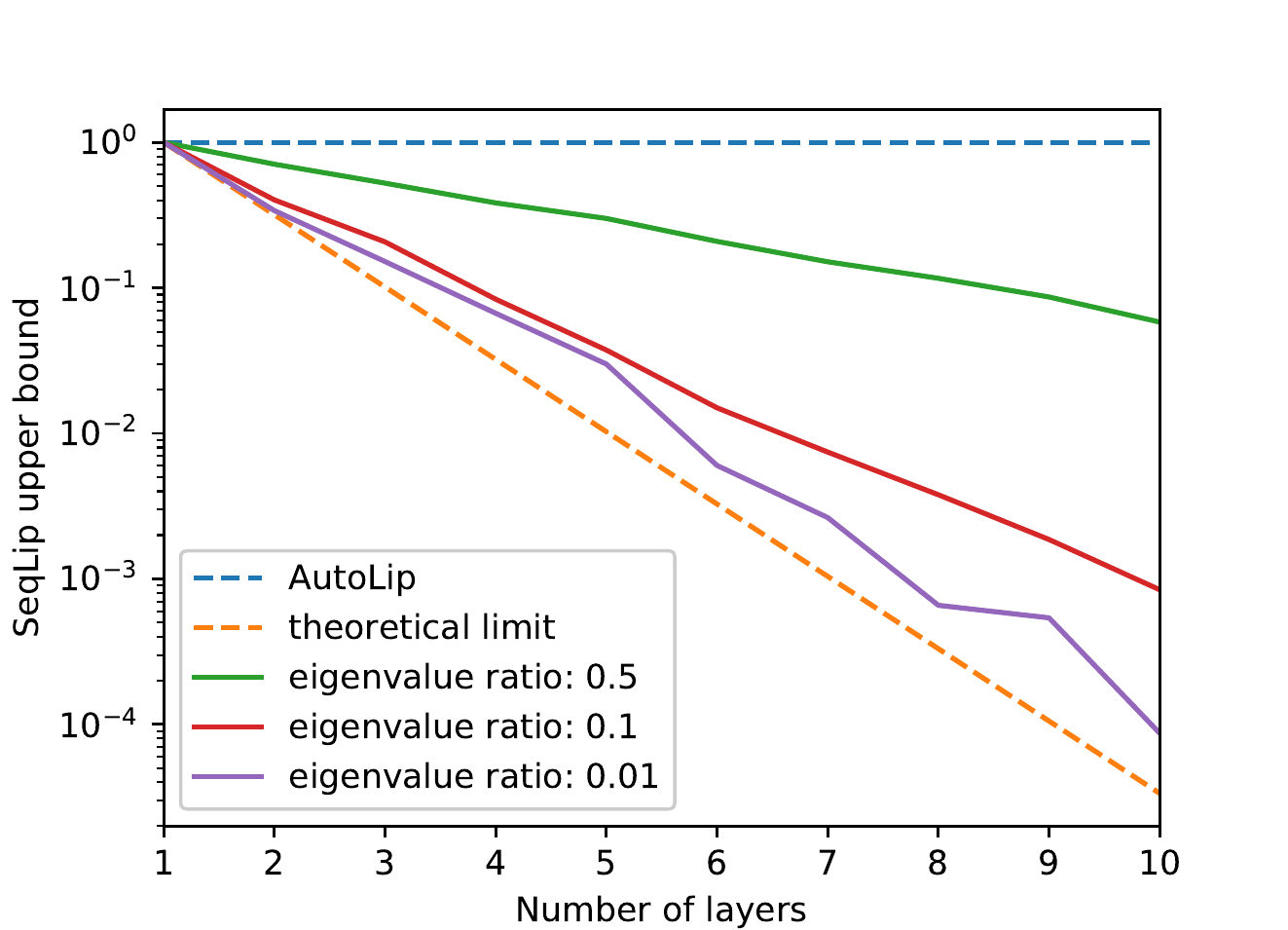}
  \caption{SeqLip in the ideal scenario.}
  \label{fig:seqlip.n.layers}
\end{minipage}%
\begin{minipage}{.5\textwidth}
  \centering
	\vspace{1.5em}
  \includegraphics[width=1\textwidth]{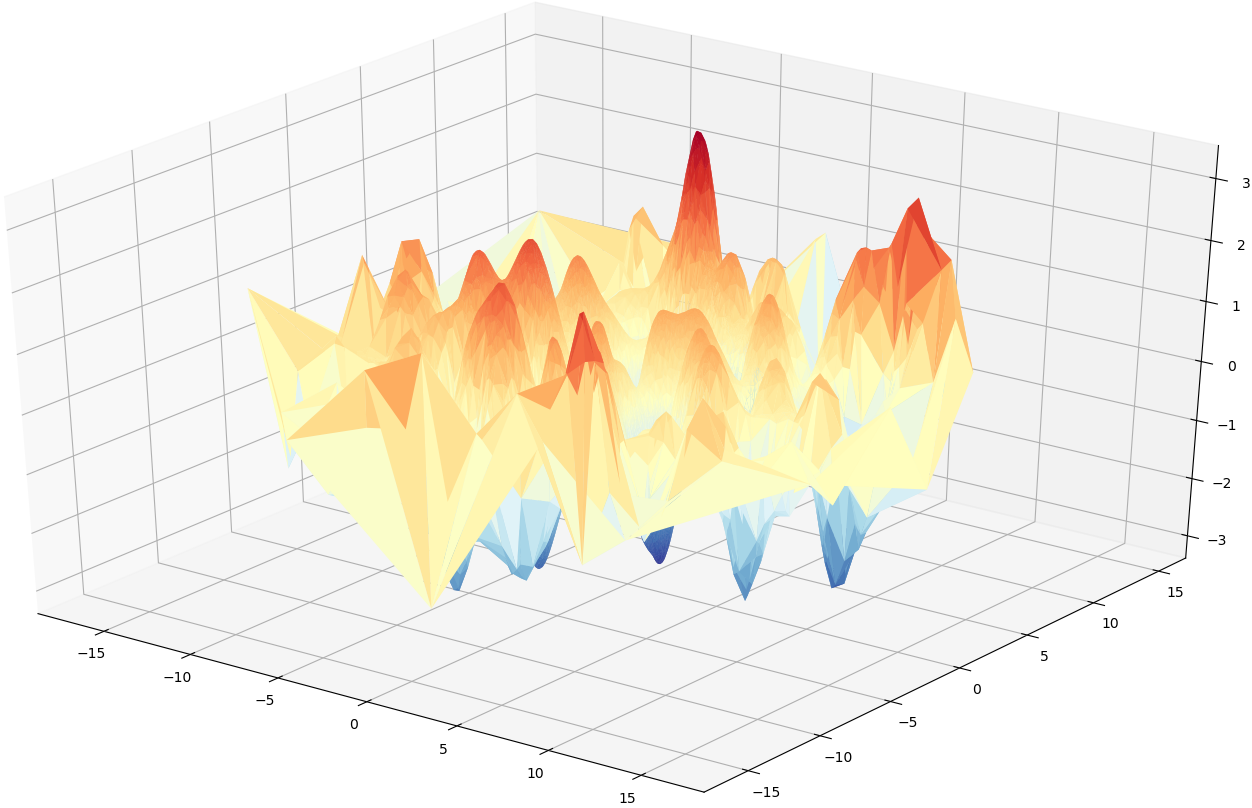}
  \caption{Synthetic function used to train MLPs.}
  \label{fig:gp.dataset}
\end{minipage}
\end{figure}

\begin{figure}[H]
  \centering
  \hspace{-12ex}
  \begin{align*}
    \begin{array}{c|cccc|ccc}
      & \multicolumn{4}{c}{\text{Upper bounds}} & \multicolumn{3}{|c}{\text{Lower bounds}}\\
      \text{\# layers} & \text{Frobenius} & \text{AutoLip} & \text{SeqLip} &\
      \text{Greedy SeqLip} & \text{Dataset} & \text{Annealing} & \text{Grid Search}\\
      \hline
      \hline
      4 & 648.2& 33.04 & 21.47 & 21.47 & 4.36 & 4.55 & 6.56\\
      5 & 4283.1 & 134.4 & 72.87 & 72.87 & 6.77 & 5.8 & 7.1\\
      7 & 22341 & 294.6 & 130.2 & 130.2 & 5.4 & 5.27 & 6.51\\
      10 & 7343800 & 19248.2 & 2463.44 & 2463.36 & 10.04 & 5.77 & 17.1
    \end{array}
  \end{align*}
  \caption{AutoLip and SeqLip for MLPs of various size.}
\end{figure}

First, since the dimension is low ($d =2$), grid search returns a very good
approximation of the Lipschitz constant, while
simulated annealing is suboptimal, probably
due to the presence of local maxima.
For upper bounds, SeqLip outperforms its competitors
reducing the gap between upper bounds and, in this case, the
true Lipschitz constant computed using grid search.

\paragraph{CNN.}

We construct simple CNNs with increasing number of layers that we train
independently on the MNIST dataset~\cite{lecun1998mnist}.
The details of the structure of the CNNs are given in the supplementary
material.
SeqLip improves by a factor of $5$ the upper bound given by AutoLip for the CNN
with $10$ layers.
Note that the lower bounds obtained with simulated annealing is probably too
low, as shown in the previous experiments.

\begin{figure}[H]
  \centering
  \hspace{-12ex}
  \begin{align*}
    \begin{array}{c|ccc|cc}
      & \multicolumn{3}{c}{\text{Upper bounds}} & \multicolumn{2}{|c}{\text{Lower bounds}}\\
      \text{\# layers} & \text{AutoLip} & \text{Greedy SeqLip} & \text{Ratio} & \text{Dataset} & \text{Annealing}\\
      \hline
      \hline
      4 & 174 & 86 & 2 & 12.64 & 25.5\\
      5 & 790.1 & 335 & 2.4 & 16.79 & 22.2\\
      7 & 12141 & 3629 & 3.3 & 31.22 & 43.6\\
      10 & 4.5\cdot 10^6 & 8.2 \cdot 10^5 & 5.4 & 38.26 & 107.8
    \end{array}
  \end{align*}
  \caption{AutoLip and SeqLip for MNIST-CNNs of various size.}
\end{figure}

\paragraph{AlexNet.}
AlexNet \citep{krizhevsky2012imagenet} is one of the first successes of deep
learning in computer vision.
The AutoLip algorithm finds that the Lipschitz constant is upper bounded by
$3.62 \times 10^7$ which remains extremely large and probably well above the
true Lipschitz constant.
As for the experiment on a CNN, we use the $200$ highest singular values of
each linear layer for Greedy SeqLip.
We obtain $5.45 \times 10^6$ as an upper bound approximation, which remains
large despite its $6$ fold improvement over AutoLip.
Note that we do not get the same results as~\cite[Section
4.3]{szegedy2013intriguing} as we did not use the same weights.

\section{Conclusion}
\label{sec:discussion}

In this paper, we studied the Lispchitz regularity of neural networks. We first
showed that exact computation of the Lipschitz constant is an $\NP$-hard
problem. We then provided a generic upper bound called \emph{AutoLip} for the
Lipschitz constant of any automatically differentiable function.
In doing so, we introduced an algorithm to compute singular values of affine
operators such as convolution in a very efficient way using \emph{autograd}
mechanism.
We finally proposed a refinement of the previous method for MLPs called \emph{SeqLip} and showed how this algorithm can improve on AutoLip theoretically and
in applications, sometimes improving up to a factor of $8$ the AutoLip upper
bound.
While the AutoLip and SeqLip upper bounds remain extremely large for neural networks of the computer vision literature (e.g. AlexNet, see \Sec{exps}), it is yet an open question to know
if these values are close to the true Lipschitz constant or substantially
overestimating it.

\subsubsection*{Acknowledgements}

The authors thank the whole team at Huawei Paris and in particular Igor Colin,
Moez Draief, Sylvain Robbiano and Albert Thomas for useful discussions and
feedback.

\bibliography{biblio}
\bibliographystyle{unsrt}

\newpage
\appendix
\section{Proof of \Theorem{lipcst.np.hardness}}
\label{appendix:lipcst.np.hardness}
We reduce the problem of maximizing a quadratic convex function on a hypercube
to $\lipcst$.
Start from the following $\NP$-hard problem~\citep[Quadratic Optimization, Section 4]{horst2013handbook}:
\begin{align}
  \label{eq:quadr.cvx}
  \begin{array}{cc}
    \displaystyle\maximize_{\sigma} & \sum_i (h_i^{\top}\sigma)^{2} = \sigma^{\top} H \sigma\\
    \st & \forall k,\ 0 \leq \sigma_k \leq 1\,,
  \end{array}
\end{align}
where $H = \sum_i h_i h_i^{\top}$ is a positive semi-definite matrix with full
rank.
Let's note
\begin{align*}
  M_1 = \left(
      \begin{array}{c|c|c|c}
        &&&\\
        h_1 & h_2 & \cdots & h_n\\
        &&&
      \end{array}
      \right)
      ,\hspace{3ex}
      M_2 = \left(
        \begin{array}{c|ccc}
          1&&&\\
          \vdots&&0&\\
          1&&&
        \end{array}
      \right)^\top\,,
\end{align*}
so that we have
\begin{align*}
  M_2 \diag(\sigma) M_1 &=
  \left(
    \begin{array}{c|ccc}
      h_1^{\top} \sigma &&&\\
      \vdots&&0&\\
      h_n^{\top} \sigma &&&\\
    \end{array}\
  \right)^\top\,.
\end{align*}
The spectral norm of this $1$-rank matrix is $\sum_i (h^{\top}_i \sigma)^2$.
We proved that \Eq{quadr.cvx} is equivalent to the following
optimization problem
\begin{align}
  \label{eq:mlp.optimization.lipschitz}
  \begin{array}{c c}
    \maximize\limits_{\sigma} &
    \|M_2 \diag(\sigma) M_1\|_2^2\\
    \st &
    \sigma \in [0, 1]^{n}\,.\\
  \end{array}
\end{align}

We recover the exact formulation of \Sec{sequential} \Eq{LipMLP} for a $2$-layer MLP (the reader
can verify there is no recursive loop).
Because $H$ is full rank, $M_1$ is surjective and all $\sigma$ are admissible
values for $g_i'(x)$ which is the equality case.
Finally, ReLU activation units take their derivative within $\{0, 1\}$ and
\Eq{mlp.optimization.lipschitz} is its relaxed optimization problem, that has
the same optimum points.

\section{Proof of \Theorem{approxSeqlip}}
Consider a single factor
$\left\| \Sigmat V \diag(\sigma) U^{\top}\Sigmat'\right\|_2$ with $V$ and
$U$ unitary matrices and $\Sigmat$ (resp. $\Sigmat'$) is diagonal with eigenvalues $(s_k)_k$
(resp. $(s'_j)_j$) in
decreasing order along the diagonal. Decompose the eigenvalue matrices as $\Sigmat =
s_1 E_{11} + D$ and $\Sigmat' = s'_1 E'_{11} + D'$, by orthogonality we can write
\begin{align}
  \left\|\Sigmat V \diag(\sigma) U^{\top}\Sigmat'\right\|_2^2 \leq&
  \Big\|s_1 E_{11} V \diag(\sigma) U^{\top} E'_{11} s'_1\\
    &+ s_1 E_{11} V_i \diag(\sigma) U^{\top} D'\nonumber\\
    &+ D\, V \diag(\sigma) U^{\top} E'_{11} s'_1\Big\|_2^2\nonumber\\
    &+ \left\|D\, V \diag(\sigma) U^{\top} D'\right\|_2^2\,.
\end{align}
First we can bound (4) $\leq (s_2 s'_2)^2$.
For (3) denote $v_k$ (resp. $u_k$) the $k$-th column of $V$
(resp. of $U$). It follows that
\begin{align*}
  \text{(3)} & \leq
  (s_1 s'_1)^2 \scprod{v_1}{\sigma \cdot u_1}^2
  + \sum_{j > 1} (s_1 s'_j)^2 \scprod{v_1}{\sigma \cdot u_j}^2
      + \sum_{k > 1} (s_k s'_1)^2 \scprod{v_k}{\sigma \cdot u_1}^2\,.
\end{align*}
The columns $(v_k)_k$ of $V$ form an orthonormal basis so we have
\begin{align*}
  \sum_{k > 1} \scprod{v_k}{\sigma \cdot u_1}^2 &= \|\sigma \cdot u_1\|^2 -
  \scprod{v_1}{\sigma \cdot u_1}^2\,,
\end{align*}
and we deduce a similar equality for
$\sum_{j > 1} \scprod{v_1}{\sigma \cdot u_j}^2$.
Using $s_k \leq s_2$ for $k > 1$ we finally obtain
\begin{align*}
  \text{(3)} & \leq
    (s_1 s'_1)^2
    \left(
      \scprod{v_1}{\sigma \cdot u_1}^2 \left(
        1 - \rt_1 - \rt_2 \right) + \rt_1 + \rt_2
    \right),
\end{align*}
with $\rt_1 = (\frac{s_2}{s_1})^2$ and $\rt_2 = (\frac{s'_2}{s'_1})^2$.
In conclusion we proved the following inequality:
\begin{align*}
  \left\|\Sigmat V \diag(\sigma) U^{\top}\Sigmat'\right\|_2^2 & \leq
    (s_1 s'_1)^2
    \left(
      \left(
        1 - \rt_1 - \rt_2 \right)\scprod{v_1}{\sigma \cdot u_1}^2  + \rt_1 + \rt_2 +
      \rt_1 \rt_2
    \right).
\end{align*}

The Lipschitz upper bound given by AutoLip of
$\left\| \Sigmat_1 V \diag(\sigma) U^{\top}\Sigmat_2\right\|_2$
is $s_{1} s'_{1}$. For the middle layers, we have
$\Sigmat = \Sigma^{1/2}$,
and the inequality still holds for the first and last layer due to $\rt_i
\leq \frac{s_2}{s_1}$; taking the maximum for $\sigma$ leads to the theorem.

\section{Proof of \Lemma{misalign}}
Let $U,V\sim\mathcal{N}(0,I_n)$ be two independent $n$-dimensional Gaussian random vectors. Then, $u = U/\|U\|_2$ and $v = V/\|V\|_2$ are uniform on the unit sphere $\mathcal{S}^{n-1}$, and
\BEQ
\BA{lll}
\vspace{1em}\displaystyle\max_{\sigma\in[0,1]^n} \left|\scprod{\sigma\cdot u}{v}\right| &=& \displaystyle\max_{\sigma\in[0,1]^n}  \left|\sum_{i=1}^n \sigma_i u_i v_i \right|\\
&=&
\displaystyle\max\left\{\sum_{i=1}^n (u_i v_i)^+, \sum_{i=1}^n (u_i v_i)^-\right\},
\EA
\EEQ
where $x^+ = \max\{0,x\}$ and $x^- = \max\{0,-x\}$ are respectively the positive and negative parts of $x$. Note that $\sum_{i=1}^n (u_i v_i)^+$ and $\sum_{i=1}^n (u_i v_i)^-$ have the same law, since the distribution of $u$ and $
v$ is symmetric \wrt the coordiante axes. Moreover, we may rewrite
\BEQ
\sum_{i=1}^n (u_i v_i)^+ = \frac{\frac{1}{n}\sum_{i=1}^n (U_i V_i)^+}{\sqrt{\frac{1}{n}\sum_{i=1}^n U_i^2}\sqrt{\frac{1}{n}\sum_{i=1}^n V_i^2}},
\EEQ
and each term converges almost surely to its expectation due to the strong law of large numbers. Finally, noting that $\Exp{U_i^2} = \Exp{V_i^2} = 1$ and
\BEQ
\Exp{(U_i V_i)^+} = \frac{1}{2}\Exp{|U_i V_i|} = \frac{1}{2}\Exp{|U_i|}\Exp{|V_i|} = \frac{1}{\pi},
\EEQ
leads to the desired result.

\section{Convolutional Neural Network of~\Sec{exps}}
\label{appendix:cnn.mnist}

For each model of depth $n$, convolution except the last one are followed by a ReLU activation
unit.

\begin{center}
\begin{tabular}{cc|cccc}
  \# Layer & Layer & \# channels out & kernel & stride & padding\\
  \hline
  \hline
  $1$ & Conv2D + bias & 32 & $(5, 5)$ & 2 & 0\\
  $2$ & Conv2D + bias & 64 & $(3, 3)$ & 2 & 0\\
  $3$ & Conv2D + bias & 64 & $(3, 3)$ & 1 & 1\\
  $\vdots$ & $\vdots$ & $\vdots$ & $\vdots$ & $\vdots$ & $\vdots$ \\
  $\vdots$ & Conv2D + bias & 64 & $(3,3)$ & 1 & 1\\
  $\vdots$ & $\vdots$ & $\vdots$ & $\vdots$ & $\vdots$ & $\vdots$ \\
  $n-1$ & Conv2D + bias & 128 & $(3, 3)$ & 2 & 0\\
  $n$ & Conv2D + bias & 10 & $(2, 2)$ & 1 & 0\\

\end{tabular}
\end{center}

\end{document}